\documentclass{article}
\usepackage{spconf,amsmath,graphicx}
\usepackage{cite}
\usepackage{amssymb,amsfonts}
\usepackage{algorithmic}
\usepackage{algorithm}
\usepackage{textcomp}
\usepackage{xcolor}
\usepackage{amsthm}
\usepackage{multicol}
\usepackage{multirow}
\usepackage{color}
\usepackage{subcaption}
\usepackage{import}

\newtheorem{mydef}{Definition}

\newtheorem{myprop}{Proposition}

\title{A ReLU Dense Layer to Improve the Performance of Neural Networks}
\name{Alireza M. Javid, Sandipan Das, Mikael Skoglund, and Saikat Chatterjee}
\address{School of Electrical Engineering and Computer Science\\
KTH Royal Institute of Technology, Sweden\\
\{almj, sandipan, skoglund, sach\}@kth.se}
\begin{document}
\ninept
\maketitle
\begin{abstract}
We propose ReDense as a simple and low complexity way to improve the performance of trained neural networks. We use a combination of random weights and rectified linear unit (ReLU) activation function to add a ReLU dense (ReDense) layer to the trained neural network such that it can achieve a lower training loss. The lossless flow property (LFP) of ReLU is the key to achieve the lower training loss while keeping the generalization error small. ReDense does not suffer from vanishing gradient problem in the training due to having a shallow structure. We experimentally show that ReDense can improve the training and testing performance of various neural network architectures with different optimization loss and activation functions. Finally, we test ReDense on some of the state-of-the-art architectures and show the performance improvement on benchmark datasets.
\end{abstract}
\section{Introduction}
\label{sec:intro}
Artificial neural networks (ANNs) and deep learning architectures have received enormous attention over the last decade. The field of machine learning is enriched with appropriately trained neural network architectures such as deep neural networks (DNNs) \cite{DNN_2013} and convolutional neural networks (CNNs) \cite{CNN_2012}, outperforming the classical methods in different classification and regression problems \cite{Russakovsky2015,DodgeK17b}. However, very little is known regarding what specific neural network architecture works best for a certain type of data. For example, it is still a mystery that for a given image dataset such as CIFAR-10, how many numbers of hidden neurons and layers are required in a CNN to achieve better performance. Typical advice is to use some rule-of-thumb techniques for determining the number of neurons and layers in a NN or to perform an exhaustive search which can be extremely expensive \cite{NodeNum_2015}. Therefore, it frequently happens that even an appropriately trained neural network performs lower than expected on the training and testing data. We address this problem by using a combination of random weights and ReLU activation functions to improve the performance of the network.

There exist three standard approaches to improve the performance of a neural network - avoiding overfitting, hyperparameter tuning, and data augmentation. Overfitting is mostly addressed by $\ell_1$ or $\ell_2$ regularization \cite{Regularization2017}, Dropout \cite{dropout2014}, and early stopping techniques \cite{Prechelt2012}. Hyperparameter tuning of the network via cross-validation includes choosing the best number of neurons and layers, finding the optimum learning rate, and trying different optimizers and loss functions \cite{yu2020hyperparameter}. And finally, the data augmentation approach tries to enrich the training set by constructing new artificial samples via random cropping and rotating, rescaling \cite{dataAugmentation2019}, and mixup \cite{zhang2018mixup}. However, all of the above techniques require retraining the network from scratch while there is no theoretical guarantee of performance improvement. Motivated by the prior uses of random matrices in neural networks as a powerful and low-complexity tool \cite{HNF2019,giryes_randomweights,mathematicsDN,Liang2020,PLN_Saikat}, we use random matrices and lossless-flow-property (LFP)  \cite{SSFN_Saikat} of ReLU activation function to theoretically guarantee a reduction of the training loss.

\textbf{Our Contribution: } We propose to add a ReDense layer instead of the last layer of an optimized network to improve the performance as shown in Figure \ref{fig:ReDense}. We use a combination of random weights and ReLU activation functions in order to theoretically guarantee the reduction of the training loss. Similar to transfer learning techniques, we freeze all the previous layers of the network and only train the ReDense layer to achieve a lower training loss. Therefore, training ReDense is fast and efficient and does not suffer from local minima or vanishing gradient problem due to its shallow structure. We test ReDense on MNIST, CIFAR-10, and CIFAR-100 datasets and show that adding a ReDense layer leads to a higher classification performance on a variety of network architecture such as MLP \cite{Schmidhuber_2015}, CNN \cite{CNN_2012}, and ResNet-50 \cite{he2015residual}.

\begin{figure*}[t!]
	%\vskip 0.2in
	\centering
    \begin{multicols}{2}
		\includegraphics[width=0.45\textwidth]{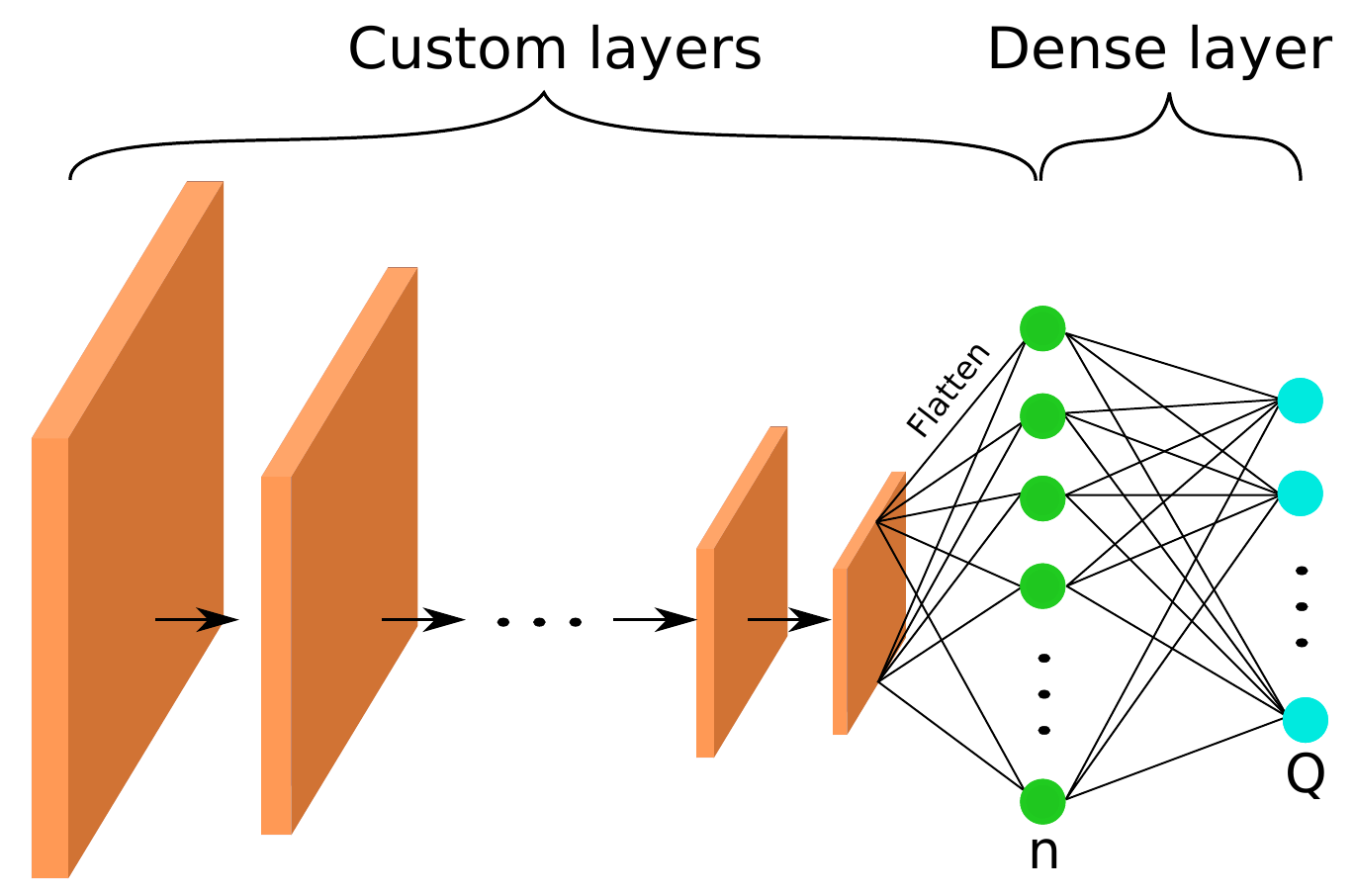}
% 		\vskip -0.05in
		\subcaption{Original network}
		\includegraphics[width=0.48\textwidth]{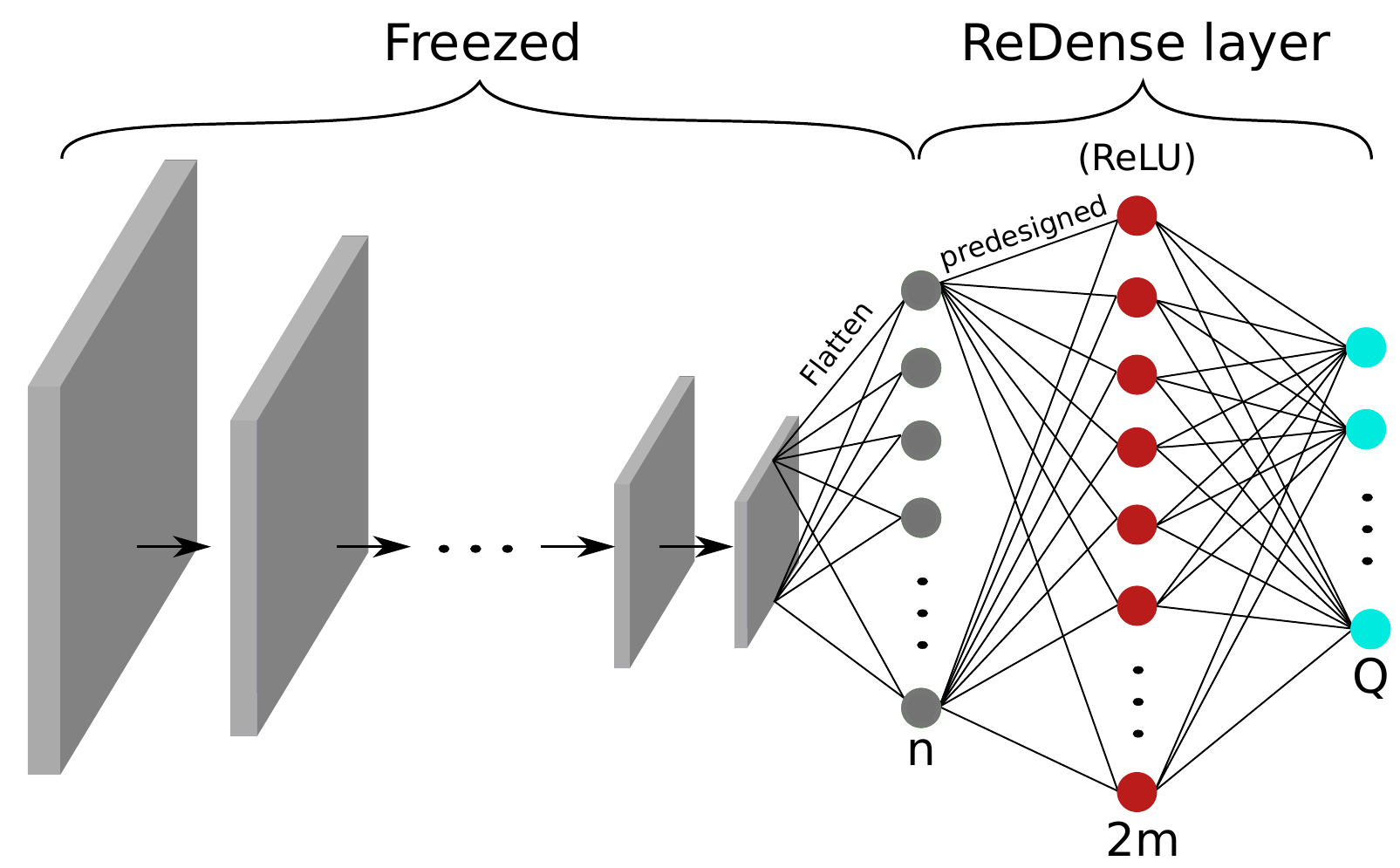}
% 		\vskip -0.05in
		\subcaption{With ReDense}
	\end{multicols}
% 	\vskip -0.25in
	\caption{Illustraition of ReDense training. ReDense can be seen as a combination of predesigned ReLU layer with $m \ge n$ and a dense layer.}
	\label{fig:ReDense}
	%\vskip -0.2in
\end{figure*}

% \section{Preliminaries}
\subsection{Artificial Neural Networks}
Consider a dataset containing $J$ samples of pair-wise $P$-dimensional input data $\mathbf{x}_j \in \mathbb{R}^{P}$ and $Q$-dimensional target vector $\mathbf{t}_j \in \mathbb{R}^{Q}$ as $\mathcal{D}=\{(\mathbf{x}_j,\mathbf{t}_j)\}_{j=1}^J$. Let us define the $j$-th feature vector at second last layer of a ANN as $\mathbf{y}_{j} = f_{\boldsymbol{\theta}}(\mathbf{x}_j) \in \mathbb{R}^{n}$, where $\boldsymbol{\theta}$ represents the parameters of all the previous layers of the network. Note that $n$ is the number of hidden neurons in the second last layer of the network.
Let us define the weights of the last layer of network by $\mathbf{O} \in \mathbb{R}^{Q \times n}$. We refer to this weight matrix as the `output weight' in the rest of the manuscript. The training of the network can be done as follows
\begin{align}
\hat{\mathbf{O}},\hat{\boldsymbol{\theta}} \in \underset{ \mathbf{O},\boldsymbol{\theta} }{ \arg\min } \,\, \sum_{j=1}^{J} \mathcal{L}(\mathbf{t}_j,\mathbf{O}f_{\boldsymbol{\theta}}(\mathbf{x}_j)) + \mathcal{R}(\mathbf{O},\boldsymbol{\theta}),
\label{eq:ANN_minimization}
\end{align}
where $\mathcal{L}$ is the loss function over a single training pair $(\mathbf{x}_j,\mathbf{t}_j)$, and $\mathcal{R}$ is the regularization term to avoid overfitting, e.g., a simple $\ell_2$-norm weight decay. Examples of loss function $\mathcal{L}$ include cross-entropy loss, mean-square loss, and Huber loss. After the training, we have access to the optimized feature vector at the second last layer. We refer to this feature vector as $\hat{\mathbf{y}}_{j} = f_{\hat{\boldsymbol{\theta}}}(\mathbf{x}_j) \in \mathbb{R}^n$ in the following sections. Note that equation \eqref{eq:ANN_minimization} is general in the sense that $f_{\boldsymbol{\theta}}(.)$ can represent any arbitrary NN architecture combined of dropout layer, convolutional layer, pooling layer, skip connections, etc. The popular optimization methods used for solving \eqref{eq:ANN_minimization} are Adagrad, ADAM, Adadelta, and their stochastic mini-batch variants \cite{ruder2017overview}.   

\subsection{Lossless Flow Property of ReLU}
We briefly discuss lossless flow property (LFP) of ReLU activation function in this section for completeness. Details can be found in \cite{SSFN_Saikat}. The proposed LFP is the key to design the ReLU layer in order to boost the performance. 
\begin{mydef}[Lossless flow property (LFP)]
	A non-linear activation function $\mathbf{g}(\cdot)$ holds the lossless flow property (LFP) if there exist two linear transformations $\mathbf{V} \in \mathbb{R}^{m \times n}$ and $\mathbf{U} \in \mathbb{R}^{n \times m}$ such that $\mathbf{U} \mathbf{g} (\mathbf{V} \mathbf{z}) = \mathbf{z}, \forall \mathbf{z} \in \mathbb{R}^n$.
\end{mydef}
This property means that the input $\mathbf{z}$ flows through the network without any loss, resulting in $\mathbf{U} \mathbf{g} (\mathbf{V} \mathbf{z}) = \mathbf{z}$. The following proposition illustrates that ReLU activation function holds LFP.
\begin{myprop}
	Consider ReLU activation $\mathbf{g}(\mathbf{z})=\max(\mathbf{z},0)$ and the identity matrix $\mathbf{I}_n \in \mathbb{R}^{n \times n}$. If 
	$
	\mathbf{V} \triangleq \mathbf{V}_n = \left[  
	\begin{array}{c}
	\mathbf{I}_n \\
	- \mathbf{I}_n
	\end{array}
	\right]  \in \mathbb{R}^{2n \times n}
	\,\, \mathrm{and} \,\, 
	\mathbf{U} \triangleq \mathbf{U}_n = \left[  
	\mathbf{I}_n  \,\, - \mathbf{I}_n
	\right] \in \mathbb{R}^{n \times 2n}
	$,
	then $\mathbf{U}_n \mathbf{g} (\mathbf{V}_n \mathbf{z}) = \mathbf{z}$ holds for every $\mathbf{z} \in \mathbb{R}^n$.
	\label{Proposition_LFP}
\end{myprop}
The above proposition can be easily proved by a simple calculation. It is shown in \cite{SSFN_Saikat} that LFP property holds for the family of ReLU-based activations such as leaky-ReLU.

\section{ReLU Dense Layer}
\label{sec:ReBoost}
Consider a trained neural network using the training dataset $\mathcal{D}=\{(\mathbf{x}_j,\mathbf{t}_j)\}_{j=1}^J$. Let us define the optimized output weight as $\hat{\mathbf{O}}$ and the corresponding feature vectors as
\begin{align}
\hat{\mathbf{y}}_{j} = f_{\hat{\boldsymbol{\theta}}}(\mathbf{x}_j) \in \mathbb{R}^n,
\label{eq:ANN_feature_vectors}
\end{align}
where $f_{\hat{\boldsymbol{\theta}}}(\cdot)$ represents the optimized preceding layers of the network. The optimized training loss can be written as
% \begin{align}
$\mathcal{L}_o=\sum_{j=1}^{J} \mathcal{L}(\mathbf{t}_j,\hat{\mathbf{O}}\hat{\mathbf{y}}_{j})$,
% \label{eq:ANN_training_loss}
% \end{align}
where $\mathcal{L}_o$ stands for old loss that we wish to improve. To this end, we want to add a new layer to the network that can provides us with a lower training loss. Let us construct a new ReLU layer as follows
\begin{align}
\bar{\mathbf{y}}_{j} = \mathbf{g} (\mathbf{V}_m \mathbf{R} \hat{\mathbf{y}}_{j}) \in \mathbb{R}^{2m},
\label{eq:ReBoost_featuree_vectors}
\end{align}
where $\mathbf{g}(\cdot)$ is ReLU activation function, $\mathbf{V}_m = \left[  
\begin{array}{c}
\mathbf{I}_m \\
- \mathbf{I}_m
\end{array}
\right]$, and $\mathbf{R} \in \mathbb{R}^{m \times n}$, $m \ge n$ is an instance of normal distribution. The training of ReDense is done as follows 
\begin{align}
\mathbf{O}^{\star} \in \underset{ \mathbf{O}}{ \arg\min } \,\, \sum_{j=1}^{J} \mathcal{L}(\mathbf{t}_j,\mathbf{O}\bar{\mathbf{y}}_{j}) \,\,\, \mathrm{s.t.} \,\,\, \|\mathbf{O}\|_F \le \epsilon,
\label{eq:ReBoost_minimization}
\end{align}
where $\epsilon$ is the regularization hyperparameter. In the following, we show that by properly choosing the hyperparameter $\epsilon$, we can achieve a lower training loss than $\mathcal{L}_o$.

\begin{myprop}
	Consider the new training loss achieved by adding the new ReLU layer $\mathcal{L}_n=\sum_{j=1}^{J} \mathcal{L}(\mathbf{t}_j,\mathbf{O}^{\star}\bar{\mathbf{y}}_{j})$, where $\mathbf{O}^{\star}$ is the solution of minimization \eqref{eq:ReBoost_minimization}. Then, there exists a choice of $\epsilon$ for which, we have $\mathcal{L}_n \le \mathcal{L}_o$.
\end{myprop}
\begin{proof}
	Let us define $\mathcal{L}_n(\mathbf{O})=\sum_{j=1}^{J} \mathcal{L}(\mathbf{t}_j,\mathbf{O}\bar{\mathbf{y}}_{j})$. We first show that there exists a point in the feasible set of \eqref{eq:ReBoost_minimization} that results in the old training loss, i.e., $\exists \bar{\mathbf{O}}, \mathcal{L}_n(\bar{\mathbf{O}}) = \mathcal{L}_o$. Consider $\bar{\mathbf{O}}=\hat{\mathbf{O}} \mathbf{R}^{\dag} \mathbf{U}_{m}$, where $\mathbf{U}_m = \left[  
	\mathbf{I}_m  \,\, - \mathbf{I}_m \right]$, and $\dag$ denotes pseudo-inverse of a matrix. Note that when $m \ge n$, the random matrix $\mathbf{R}$ is full-column rank and therefore, its pseudo-inverse exits. Then, by using LFP in Proposition \ref{Proposition_LFP} and the feature vector in \eqref{eq:ReBoost_featuree_vectors}, we have 
	\begin{align}
	\mathcal{L}_n(\bar{\mathbf{O}}) & = \sum_{j=1}^{J} \mathcal{L}(\mathbf{t}_j,\hat{\mathbf{O}} \mathbf{R}^{\dag} \mathbf{U}_{m}\bar{\mathbf{y}}_{j}) \nonumber \\
	& = \sum_{j=1}^{J} \mathcal{L}(\mathbf{t}_j,\hat{\mathbf{O}} \mathbf{R}^{\dag} \mathbf{U}_{m}\mathbf{g} (\mathbf{V}_m \mathbf{R} \hat{\mathbf{y}}_{j})) \nonumber \\
	& = \sum_{j=1}^{J} \mathcal{L}(\mathbf{t}_j,\hat{\mathbf{O}} \mathbf{R}^{\dag} \mathbf{R} \hat{\mathbf{y}}_{j})) \nonumber \\
	& = \sum_{j=1}^{J} \mathcal{L}(\mathbf{t}_j,\hat{\mathbf{O}} \hat{\mathbf{y}}_{j})) = \mathcal{L}_o.
	\end{align}
	Therefore, by choosing $\epsilon= \| \hat{\mathbf{O}} \mathbf{R}^{\dag} \mathbf{U}_{m} \|_F$, we make sure to have the smallest feasible set possible to avoid overfitting while including $\bar{\mathbf{O}}$ in the search space. On the other hand, we know that $\mathcal{L}_n(\mathbf{O}^{\star}) \le \mathcal{L}_n(\bar{\mathbf{O}})$ by definition in \eqref{eq:ReBoost_minimization}, which concludes the proof.
\end{proof}

Note that the optimization in \eqref{eq:ReBoost_minimization} is not computationally complex as it only minimizes the total loss with respect to the output weight $\mathbf{O}$. The weight matrix of the ReLU layer in \eqref{eq:ReBoost_featuree_vectors} is fixed before training, and there is no need for error backpropagation. Therefore, by solving \eqref{eq:ReBoost_minimization}, we have added a ReLU dense layer to a trained network to reduce its training loss. 

Detailed steps of training ReDense are outlined in Algorithm \ref{algorithm}. Note that we have not included the update step of learning rate $\eta_t$ in Algorithm \ref{algorithm} for the sake of simplicity. In practice, training of ReDense can be done via any of the common adaptive gradient descent variants such as ADAM, Adagrad, etc. An illustration of ReDense training is shown in Figure \ref{fig:ReDense}, where the grayscale layers indicate being freezed during the training. Therefore, ReDense is not prone to vanishing gradient issue and is computationally fast. Note that the original network can be any arbitrary neural network architecture, making ReDense a universal technique to reduce the training loss.

\begin{algorithm}[t!]
	\caption{: ReDense Training}
	\label{algorithm}
	\mbox{Input: }
	\begin{algorithmic}[1]
		\STATE $\hat{\mathbf{O}}$: output weight matrix of the old trained neural network
		\STATE $\{(\hat{\mathbf{y}}_j, \mathbf{t}_j)\}_{j=1}^{J}$: training features of the output layer and their corresponding target vectors
		\STATE $\eta_0$: initial learning rate of the gradient descent optimizer
		\STATE $\tau$: number of epochs of the gradient descent optimizer
	\end{algorithmic}
	\mbox{Initialization:}
	\begin{algorithmic}[1]
		\STATE Construct the ReLU layer as $\bar{\mathbf{y}}_{j} = \mathbf{g} (\mathbf{V}_m \mathbf{R} \hat{\mathbf{y}}_{j})$ 
		\STATE $\epsilon \leftarrow  \| \hat{\mathbf{O}} \mathbf{R}^{\dag} \mathbf{U}_{m} \|_F$
		\STATE $\mathbf{O}_0 \leftarrow \hat{\mathbf{O}} \mathbf{R}^{\dag} \mathbf{U}_{m}$
		\STATE $t \leftarrow 0$
	\end{algorithmic}
	\mbox{Training:}
	\begin{algorithmic}[1]
		\REPEAT 
		\STATE $\mathbf{G}_t \leftarrow \sum_{i=1}^{N} \nabla_{\mathbf{O}} \mathcal{L}(\mathbf{t}_i,\mathbf{O}_t\bar{\mathbf{y}}_{i})$
		\STATE $\mathbf{Z} \leftarrow \mathbf{O}_{t} - \eta_t \mathbf{G}_t $
		\IF{$\|\mathbf{Z}\|_F \geq \epsilon$}
		\STATE $\mathbf{O}_{t+1} \leftarrow \mathbf{Z}.(\frac{\epsilon}{\|\mathbf{Z}\|_F})$
		\ENDIF
		\STATE $t \leftarrow t+1$
		\UNTIL  $t = \tau$ 
	\end{algorithmic}
\end{algorithm}

\section{Experimental Results}
In this section, we carry out several experiments on a variety of neural network architectures to show the performance capabilities of ReDense. First, we construct a simple MLP and investigate the effects of changing the number of random hidden neurons $m$ in ReDense. Then, we illustrate the performance of ReDense on more complicated models such as convolutional neural networks with various types of loss functions and output activations. Finally, we choose some of the state-of-the-art networks on benchmark classification datasets and try improving their performance via ReDense. Our open-sourced code can be found at \url{https://github.com/alirezamj14/ReDense}.

\subsection{Multilayer Perceptrons}
\label{subsec:DNN}
In this section, we use a simple two-layer feedforward neural network with $n=500$ hidden neurons to train on the MNIST dataset. The network is trained for 100 epochs with respect to softmax cross-entropy loss by using ADAM optimizer with a batch size equal to 128 and a learning rate of $10^{-6}$. This combination leads to test classification accuracy of $93.1\%$ on MNIST. By using the trained network, we construct the feature vector $\bar{\mathbf{y}}_j$ according to \eqref{eq:ReBoost_featuree_vectors} for different choices of $m = 500, 1000, 1500, 2000$ and add a ReDense layer to the network. Note that here we must have $m \ge n = 500$.

Figure \ref{fig:MLP_different_n} shows the learning curves of ReDense for different values of $m$. Note that the initial performance at epoch $0$ is, in fact, the final performance of the above MLP with $93.1\%$ testing accuracy. We use full-batch ADAM optimizer with a learning rate of $10^{-5}$ for ReDense in all cases. By comparing the curves of training and testing loss, we observe that ReDense consistently reduces both training and testing loss while maintaining a small generalization error. Interestingly, ReDense achieves a higher testing accuracy for $m=500$ and improves the testing accuracy by $2.4\%$. The reason for this behaviour is that increasing $m$, in fact, reduces the value of $\epsilon= \| \hat{\mathbf{O}} \mathbf{R}^{\dag} \mathbf{U}_{m} \|_F$ in ReDense, leading to a tighter overfitting constraint as the dimensions of the feature vectors increases. One can, of course, manually choose larger values of $\epsilon$ for ReDense but we found in our experiments that a looser $\epsilon$ leads to poor generalization. Therefore, we use the smallest possible value of $m = n$ in the rest of the experiments.

\begin{figure*}[t!]
	%\vskip 0.2in
	\centering
	\begin{multicols}{3}
		\includegraphics[width=0.32\textwidth]{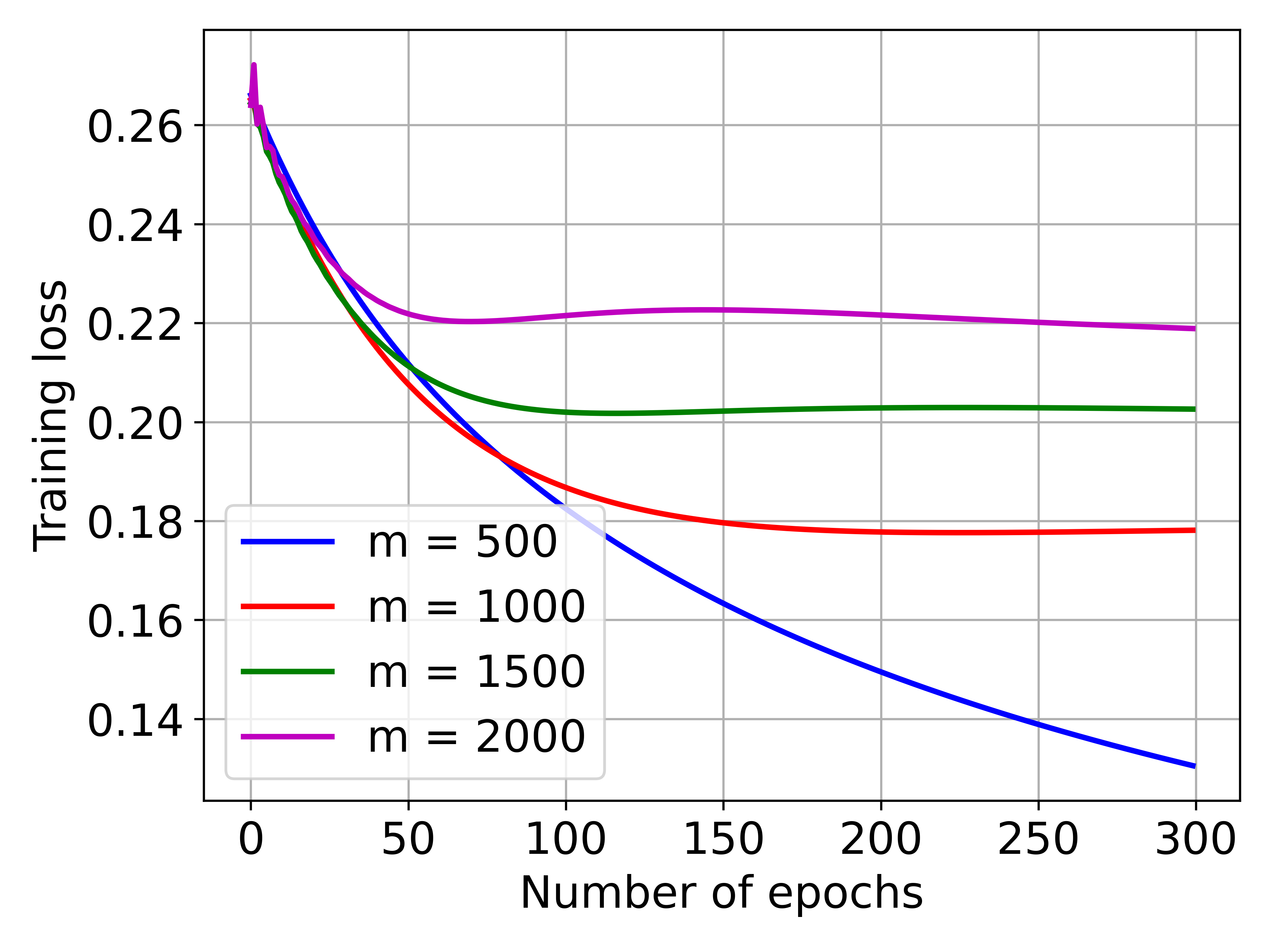}
% 		\vskip -0.05in
		\subcaption{Training loss}
		\includegraphics[width=0.32\textwidth]{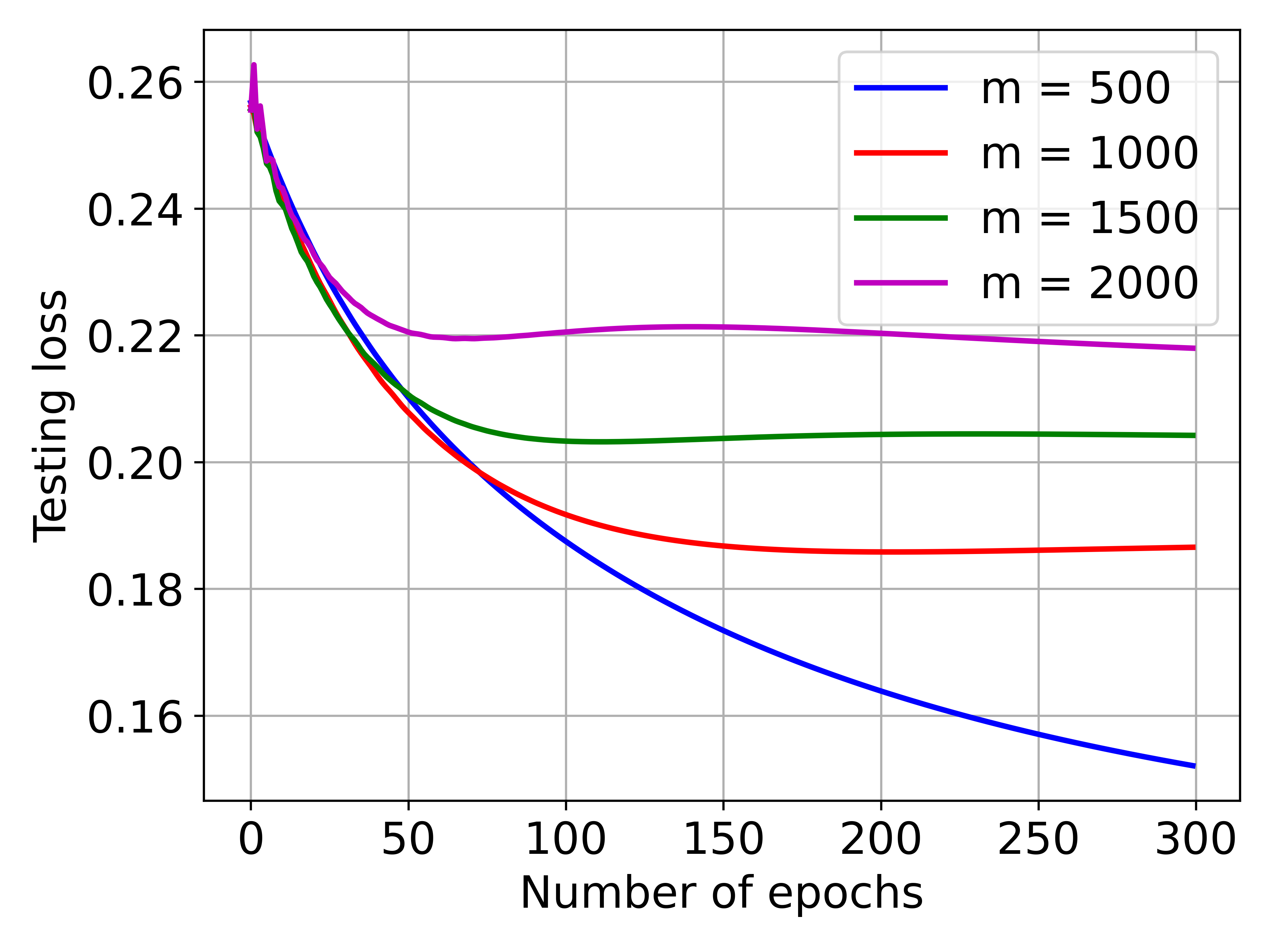}
% 		\vskip -0.05in
		\subcaption{Testing loss}
		\includegraphics[width=0.32\textwidth]{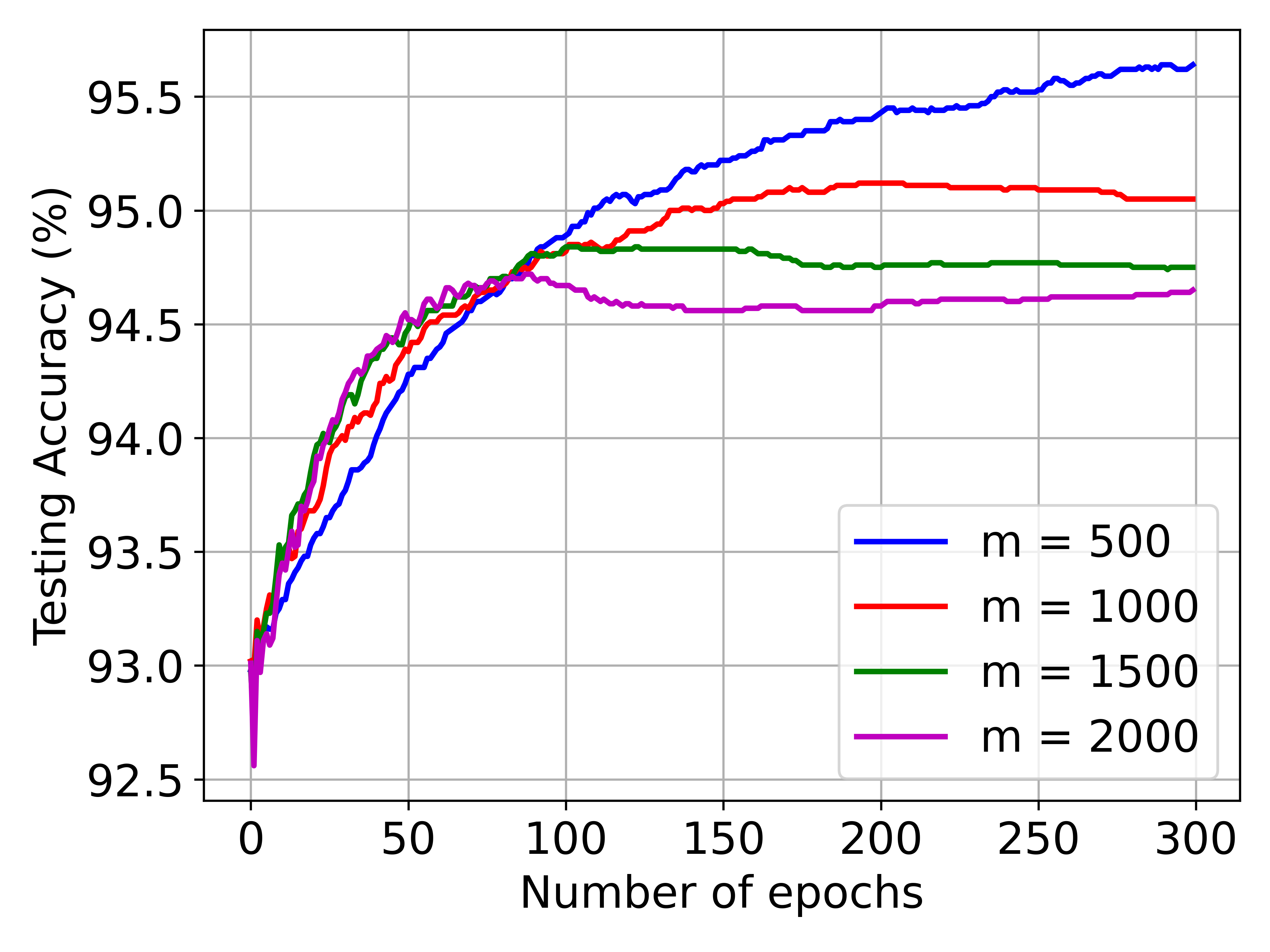}
% 		\vskip -0.05in
		\subcaption{Testing accuracy}
	\end{multicols}
% 	\vskip -0.25in
	\caption{Performance improvement versus number of epochs for a MLP trained on MNIST by using softmax cross-entropy loss.}
	\label{fig:MLP_different_n}
	%\vskip -0.2in
\end{figure*}

\subsection{Convolutional Neural Networks}
\label{subsec:CNN}
In this section, we use a convolutional neural network with the following architecture: two Conv2D layer with 32 filters of size $3 \times 3$ followed by a MaxPooling2D layer of size $3 \times 3$, dropout layer with $p=0.25$, Conv2D layer with 64 filters of size $3 \times 3$, Conv2D layer with 32 filters of size $3 \times 3$, MaxPooling2D layer of size $3 \times 3$, dropout layer with $p=0.25$, flatten layer, dense layer of size 512, dropout layer with $p=0.5$, and dense layer of size 10. We train this network on the CIFAR-10 dataset using ADAM optimizer with a batch size of 128 for 15 epochs with respect to different loss functions. We use 10 percent of the training sample as a validation set. Note that in this case $n=512$ is the dimension of the feature vector $\hat{\mathbf{y}}_j$. We choose the smallest possible value for the number of random hidden neurons of ReDense $m = 512$ as it is shown to be the best choice in Section \ref{subsec:DNN}. 

Figure \ref{fig:CNN_different_loss} shows the learning curves of ReDense for different choices of CNN loss functions, namely, cross-entropy (CE) loss, mean-square-error (MSE) loss, Poisson loss, and Huber loss. Note that the initial performance at epoch $0$ is, in fact, the final performance of the above CNN in each case. We use full-batch ADAM optimizer with a learning rate of $10^{-4}$ for ReDense in all cases. By comparing the curves of training and testing loss, we observe that ReDense consistently reduces both training and testing loss while maintaining a small generalization error. It can be seen that ReDense approximately improves testing accuracy by $1$-$2\%$ in all cases. This experiment shows that ReDense can improve the performance of neural networks regardless of the choice of training loss. Note that we have used softmax cross-entropy loss in all the experiments for training ReDense.  

\begin{figure*}[t!]
	%\vskip 0.2in
	\centering
	\begin{multicols}{3}
		\includegraphics[width=0.32\textwidth]{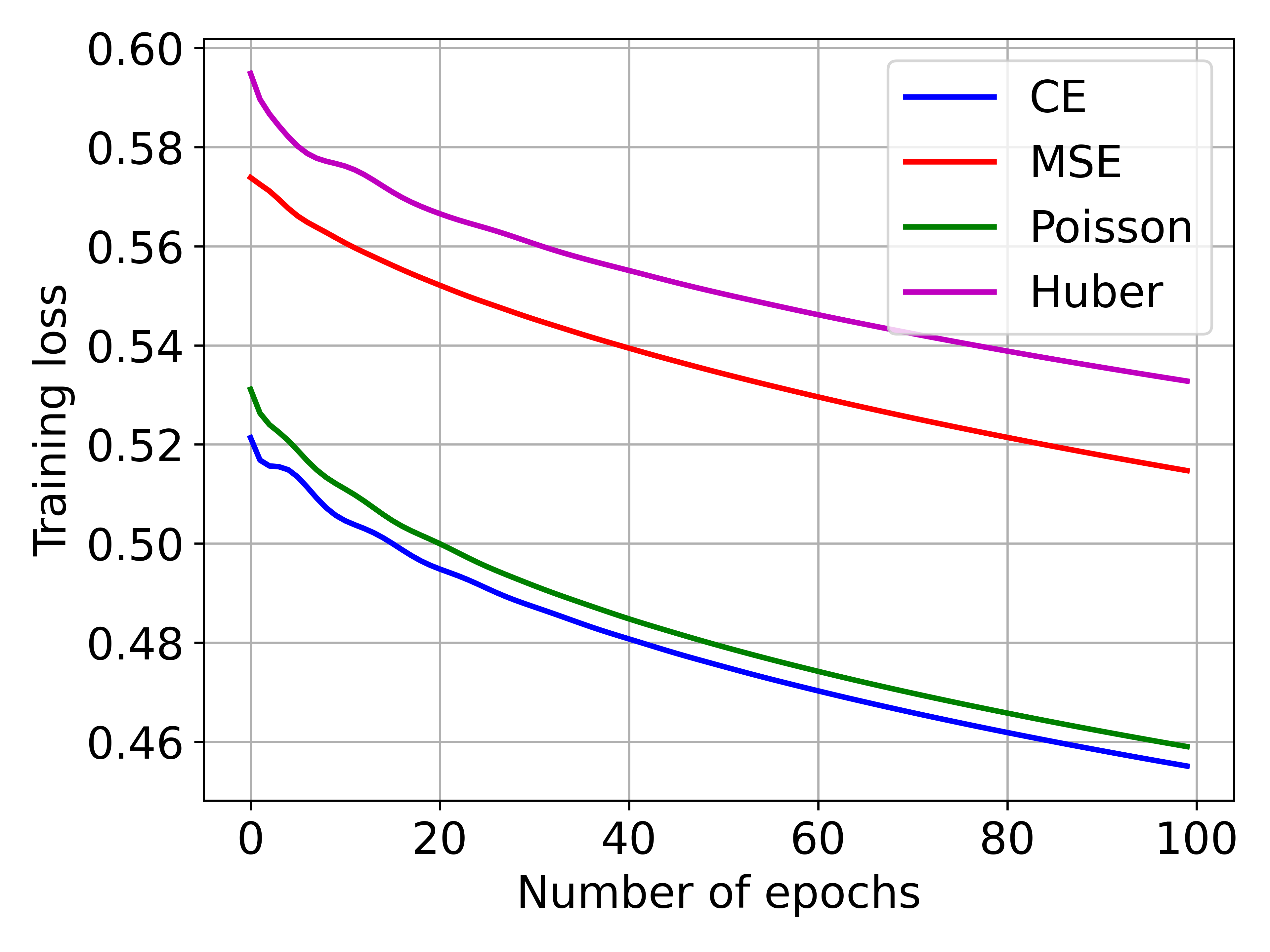}
% 		\vskip -0.05in
		\subcaption{Training loss}
		\includegraphics[width=0.32\textwidth]{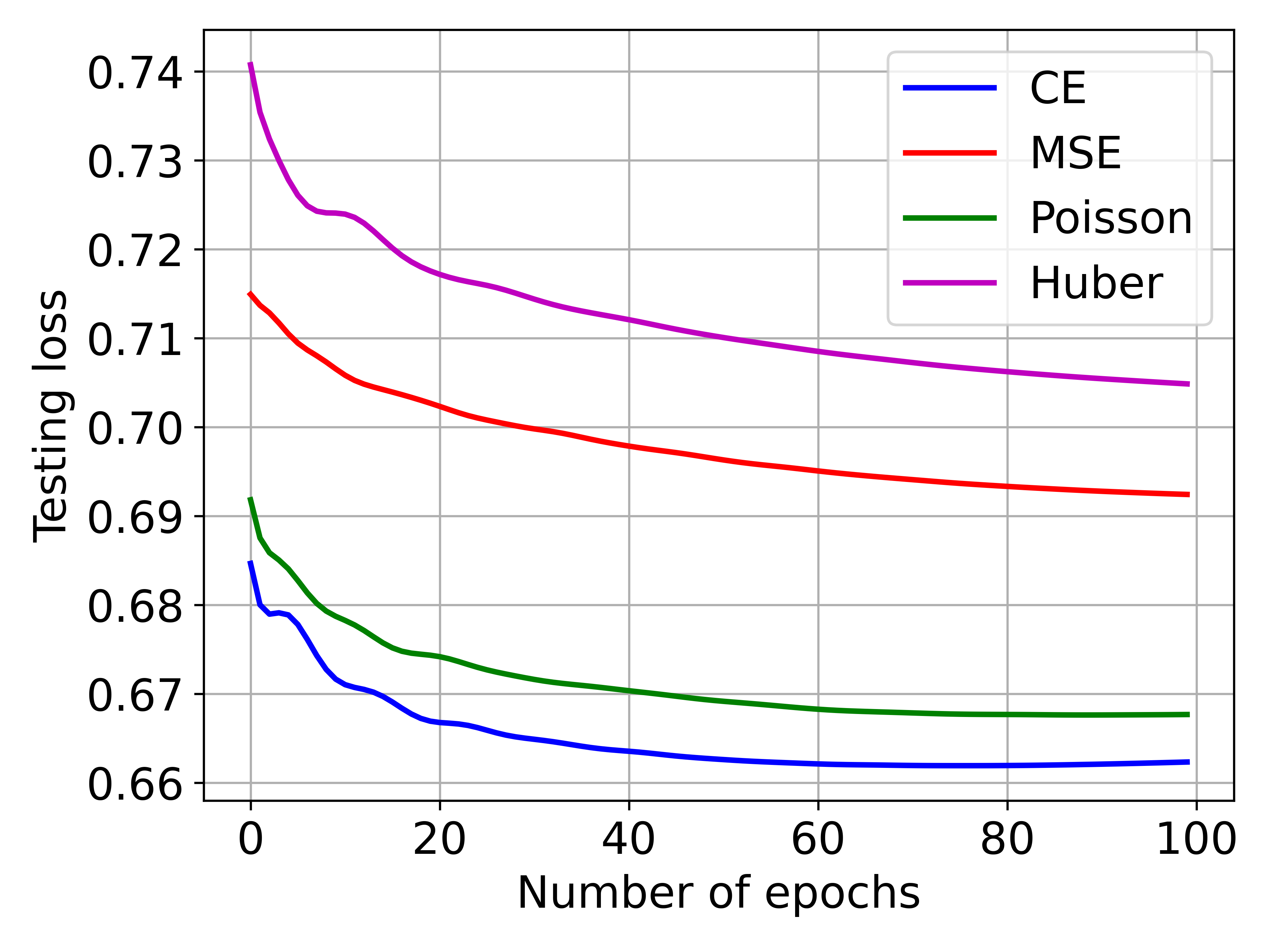}
% 		\vskip -0.05in
		\subcaption{Testing loss}
		\includegraphics[width=0.32\textwidth]{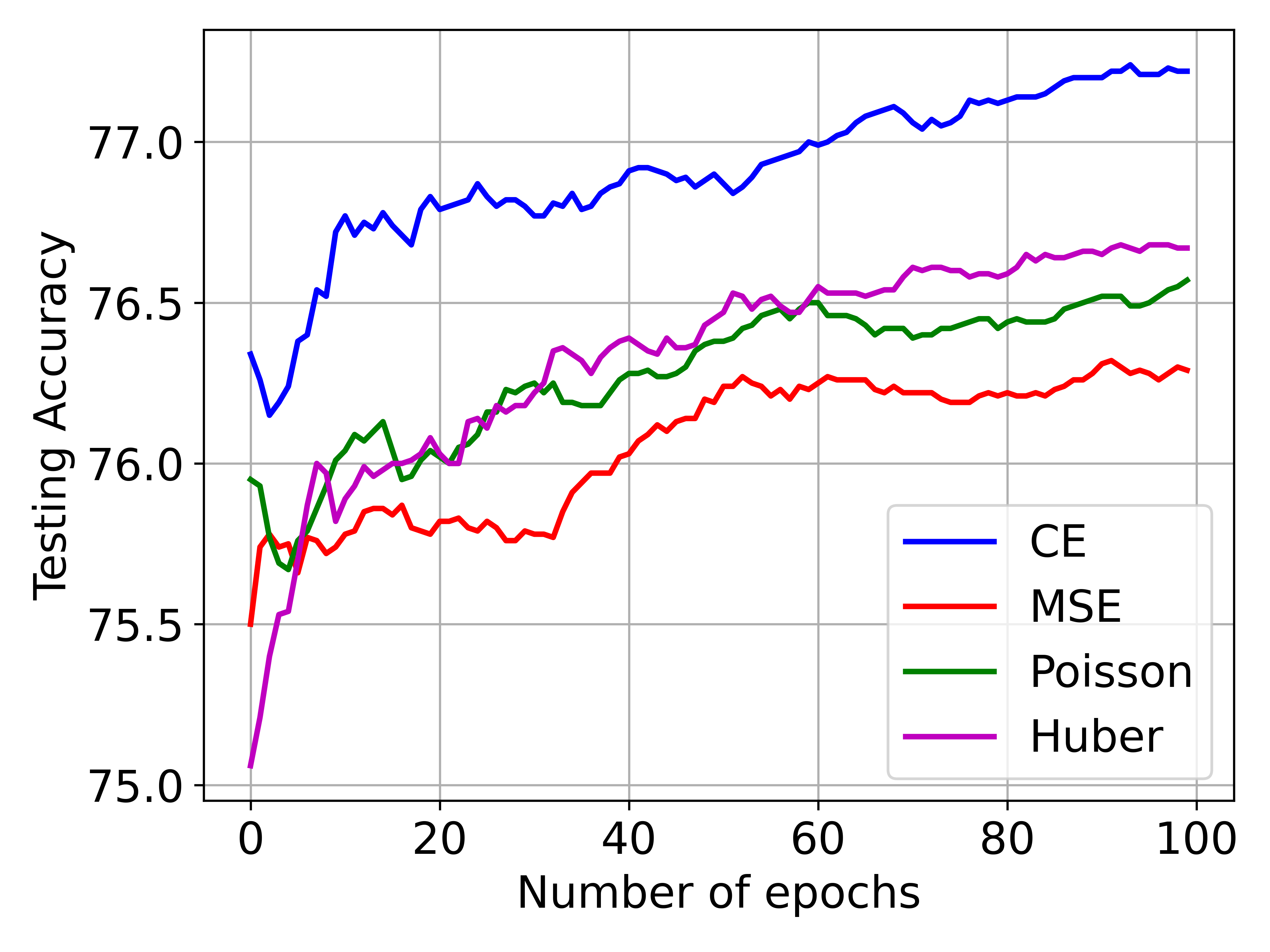}
% 		\vskip -0.05in
		\subcaption{Testing accuracy}
	\end{multicols}
% 	\vskip -0.25in
	\caption{Performance improvement versus number of epochs for a CNN trained on CIFAR-10 by using different loss functions.}
	\label{fig:CNN_different_loss}
	%\vskip -0.2in
\end{figure*}

\subsection{State-of-the-art Networks}
\label{subsec:Boost_SOTA}
We apply ReDense on the best performer architecture on CIFAR datasets to observe the improvement in classification accuracy. To the best of our knowledge, Big Transfer (BiT) \cite{kolesnikov2020big} models are the top performer on CIFAR datasets in the literature. BiT uses a combination of fine-tuning techniques and transfer learning heuristics to achieve state-of-the-art performance on a variety of datasets. BiT-M-R50x1 model uses a ResNet-50 architecture, pretrained on the public ImageNet-21k dataset, as the baseline model and by applying the BiT technique, achieves testing accuracy of $95.35\%$ and $82.64\%$ on CIFAR-10 and CIFAR-100, respectively.  

\begin{table}[t!]
		\centering
		\caption{Testing accuracy ($\%$) with and without ReDense}
% 		\vskip -0.05in
		\label{table:Classification_accuracy}
		\setlength{\tabcolsep}{10pt}
		\renewcommand{\arraystretch}{1.2}
		\begin{tabular}{|c|c|c|c|}
			\hline
		    \multirow{2}{*}{Dataset} & \multicolumn{2}{|c|}{BiT-M-R50x1} \\ \cline{2-3}
			& with ReDense & without \\
			\hline \hline
			CIFAR-10 & \textbf{96.61} & 95.35 \\
			\hline
			CIFAR-100 & \textbf{84.28} & 82.64 \\ 
			\hline
		\end{tabular}
\end{table} 

We add a ReDense layer with $m = 2048$ random hidden neurons since the dense layer of BiT-M-R50x1 model has a dimension of $n=2048$. We use full-batch ADAM optimizer with a learning rate of $10^{-4}$ and train ReDense layer with respect to softmax cross-entropy loss for 100 epochs. We observe that BiT-M-R50x1 + ReDense achieves testing accuracy of $96.61\%$ and $84.28\%$ on CIFAR-10 and CIFAR-100, respectively, as shown in Table \ref{table:Classification_accuracy}. This experiment illustrates the capability of ReDense in improving the performance of highly optimized and deep networks. It is worth noting that ReDense is trained within a few minutes on a laptop while training BiT-M-R50x1 required several hours in a GPU cloud.

However, the power of ReDense is limited to scenarios where the network is not overfitted with a $100\%$ training accuracy since ReDense tries to reduce the training loss. For example, we also tested BiT-M-R101x3 model which uses a ResNet-101 as the baseline model, and achieves $98.3\%$ testing accuracy on CIFAR-10 but with a training accuracy of $100\%$. After adding the ReDense layer, we observed no tangible improvement as the network was already overfitted and there was no room for improvement. Early stopping of the baseline model, including dropout, data augmentation, and unfreezing some of the previous layers of the network are among the possible options to avoid such extreme overfitting in future works.

\section{Conclusion}
We showed that by adding a ReDense layer, it is possible to improve the performance of various types of neural networks including the state-of-the-art if the network is not overfitted. ReDense is a universal technique in the sense that it can be applied to any neural network regardless of its architecture and loss function. Training of ReDense is simple and fast due to its shallow structure and does not require retraining the network from scratch. A potential extension of this work would be to apply ReDense along with other performance improvement techniques such as dropout and data augmentation to further increase the predictive capabilities of neural networks. Besides, ReDense can also be applied in any of the intermediate layers of a network to reduce the training loss.  

\bibliographystyle{IEEEbib}
\bibliography{javid}

\end{document}